\def\year{2023}\relax

\documentclass[letterpaper]{article}
\usepackage{lmodern}
\usepackage{amssymb,amsmath}

\usepackage{aaai23}
\usepackage{times}
\usepackage{helvet}
\usepackage{courier}
\usepackage{graphicx}
\usepackage{algorithm}
\usepackage{algpseudocode}
\usepackage{xcolor}
\usepackage{amsthm}
\usepackage{color}
\usepackage{verbatim}
\usepackage{tikz}
\usepackage{pgfplots}
\usetikzlibrary{arrows.meta,shapes}
\usetikzlibrary{pgfplots.groupplots}
\usepackage{multirow}
\usepackage{enumerate}
\usepackage{booktabs}
\usepackage{nicefrac}
\usepackage{dblfloatfix}
\usepackage{enumitem}
\usepackage{natbib}

\newtheorem{thm}{Theorem}
\newtheorem{lemma}[thm]{Lemma}

\newtheorem{definition}[thm]{Definition}
\newtheorem{corollary}[thm]{Corollary}

\usepackage[margin=1in]{geometry}
\usepackage{hyperref}
\hypersetup{unicode=true,
            pdftitle={On the Completeness of Conflict-Based Search: Temporally-Relative Duplicate Pruning},
            pdfauthor={Thayne T. Walker; Nathan R. Sturtevant},
            pdfborder={0 0 0},
            breaklinks=true}
\usepackage{graphicx,grffile}

\newcommand{\wait}{\boldsymbol{\circlearrowleft}}

\title{On the Completeness of Conflict-Based Search: \\ Temporally-Relative Duplicate Pruning}

\author{
Thayne T. Walker$^{1,2}$, Nathan R. Sturtevant$^3$ \\
$^1$University of Denver, Denver, USA 
$^2$Lockheed Martin Corporation, USA \\
$^3$Department of Computing Science, Alberta Machine Intelligence Institute (Amii), University of Alberta, Canada \\
thayne.walker@du.edu, nathanst@ualberta.ca
}
\begin{document}
\maketitle
\begin{abstract}
A well-known and often-cited deficiency of the Conflict-Based Search (CBS) algorithm for the multi-agent pathfinding (MAPF) problem is that it is incomplete for problems which have no solution; if no mitigating procedure is run in parallel, CBS will run forever when given an unsolvable problem instance. In this work, we introduce Temporally-Relative Duplicate Pruning (TRDP), a technique for duplicate detection and removal in both classic and continuous-time MAPF domains. TRDP is a simple procedure which closes the long-standing theoretic loophole of incompleteness for CBS by detecting and avoiding the expansion of duplicate states. TRDP is shown both theoretically and empirically to ensure termination without a significant impact on runtime in the majority of problem instances. In certain cases, TRDP is shown to increase performance significantly.
\end{abstract}

\section{Introduction}

The objective of multi-agent pathfinding (MAPF)~\cite{stern2019multi} is to find paths for multiple agents to move from their current configurations (or states) to goal states such that their respective paths do not conflict at any time. A \emph{conflict} occurs when agents' shapes overlap at the same time. In contrast to ``classic'' MAPF with actions of unit duration, real world domains often require the use of continuous-time, non-unit duration actions. In this paper, we seek optimal, conflict-free solutions to the MAPF problem in both classic MAPF and continuous-time MAPF, also known as MAPF\textsubscript{R}~\cite{walker2018extended}.

Optimal MAPF solvers are classified into two broad categories: \emph{coupled} and \emph{decoupled}. Coupled algorithms such as multi-agent A* (MA-A*)~\cite{ODID}, enhanced partial expansion A*~\cite{goldenberg2014epea} and M*~\cite{MSTAR} solve MAPF problems in a joint state space, where the joint states of all agents are aggregated into a single state. Decoupled algorithms such as CBS~\cite{CBS}, ICTS~\cite{ICTS}, CBICS~\cite{walker2021conflict}, branch-and-cut-and-price (BCP)~\cite{lam2019branch} and enhanced variants of these~\cite{gange2019lazy,MDDSAT,li2019disjoint} solve MAPF problems without aggregating the agents' state spaces together, or by partially aggregating only some of them. In general, decoupled algorithms have a lower practical computational complexity and hence are more popular. CBS is a popular MAPF algorithm. In this paper we focus on the completeness of CBS, but our approach (especially for duplicate detection) is applicable to most coupled and decoupled algorithms for MAPF\textsubscript{R}.

A long-standing problem for CBS is that it will run forever if given an unsolvable problem instance~\cite{CBS}. The suggested remedy for this is to run a sub-optimal, polynomial-time, complete algorithm~\cite{kornhauser1984coordinating,pps,okumura2022lacam} in parallel with the main solver. In the case that no solution exists, the polynomial-time algorithm will report this fact, and then CBS can be terminated. But this reliance on a second MAPF algorithm means that CBS is not natively complete. When using CBS for domains other than classic MAPF (e.g., MAPF\textsubscript{R}), it may be difficult to find or invent a separate, complete algorithm for the desired domain. At this time, no complete algorithms generally applicable to MAPF\textsubscript{R} are known. In summary, there is a need for making CBS \emph{natively} complete.

Temporally-relative duplicate pruning (TRDP) is a novel technique for CBS which can guarantee completeness for classic MAPF and MAPF\textsubscript{Q}, a subset of MAPF\textsubscript{R} with discrete rational values. TRDP renders otherwise infinite CBS search spaces finite by detecting and eliminating multi-agent ``loops''. TRDP can be applied to virtually any MAPF domain to guarantee completeness. Importantly, the theory of TRDP resolves a long-standing deficiency of CBS. Furthermore, our empirical results show that TRDP correctly detects and terminates on unsolvable MAPF problem instances while having only a small effect on the runtime in a representative set of (solvable) classic MAPF and MAPF\textsubscript{Q} instances.

\section{Problem Definition}
\label{sec:def}

MAPF was originally defined for a ``classic'' setting~\cite{stern2019multi} where the movements of agents are coordinated on a two dimensional grid, usually represented as a graph $G{=}(V,E)$. In classic MAPF, edges have a unit time duration and agents occupy a point in space. This paper uses the definition of MAPF\textsubscript{R}~\cite{walker2018extended}, a variant of MAPF for continuous time execution where $G$ is a weighted graph which (unlike grids) may be non-planar, meaning edges may intersect in areas other than vertices. Every vertex $v{\in}V$ has coordinates in a metric space and every edge $e{\in}E$ has a positive real-valued edge weight $w(e){\in}\mathbb{R}_+$. This includes self-directed edges for \emph{wait} actions.
MAPF\textsubscript{Q}$\subset$MAPF\textsubscript{R} uses only positive, rational-valued edge weights $w(e){\in}\mathbb{Q}_+$. Weights usually represent the times it takes to traverse edges, but cost and time duration can be treated separately. In MAPF there are $k$ agents and each agent is assigned a start and a goal vertex $V_s{=}\{start_1,..,start_k\}{\subseteq}V$ and $V_g{=}\{goal_1,..,goal_k\}{\subseteq}V$ such that $start_i{\neq} start_j$ and $ goal_i{\neq}goal_j$ for all $i{\neq}j$.

A \emph{solution} to a MAPF\textsubscript{R} (MAPF\textsubscript{Q}) instance is $\Pi{=}\{\pi_1,..,\pi_k\}$, a set of single-agent \emph{paths} composed of \emph{states}. A state $s{=}(v,t)$ is composed of a vertex $v{\in}V$ and a time $t{\in}\mathbb{R}_+$ ($t{\in}\mathbb{Q}_+$).
A single-agent path is a sequence of $d{+}1$ states $\pi_i{=}[s_i^0,..,s_i^d]$, where $s_i^0{=}(start_i,0)$ and $s_i^d{=}(goal_i,t_g)$ where $t_g$ is the time the agent arrives at its goal and all vertices in the path traverse edges in $E$. Costs and time steps are always monotonically increasing in a path because edge costs are always greater than zero.

Agents have a shape, (e.g., a circle or polygon), which is situated relative to a \emph{reference point}~\cite{li2019large}. Agents move along edges from a vertex to an adjacent vertex, and could use constant or variable velocity in the metric space. 
A \emph{conflict} happens when two agents perform actions (by either waiting or traversing edges) which results in their shapes overlapping simultaneously.  We seek a \emph{feasible solution}, which has no conflicts between any pairs of paths in $\Pi$. MAPF solvers typically optimize for minimum makespan or minimum flowtime.
Optimization of classic MAPF, MAPF\textsubscript{Q} and MAPF\textsubscript{R} is NP-hard~\cite{YuLaValle2013}.

\section{Background and Motivation}

Decoupled, search-based algorithms for MAPF\textsubscript{R} include sub-optimal ones~\cite{Silver05,direction,AASIPP,cohen2019optimal} and optimal ones~\cite{walker2018extended,andreychuk2019continuous,walker2020generalized,walker2021conflict,andreychuk2021improving,coppe2022conflict,walker2022dissertation}, but none of them will terminate when no solution exists for the problem instance.
\label{sec:motTRDP}
Recall that \emph{completeness} means that an algorithm is guaranteed to terminate in a finite amount of time. There are two parts to completeness~\cite{edelkamp2011heuristic}:
\begin{enumerate}
    \item Termination with a solution if one exists.
    \item Termination with $\emptyset$ if a solution does not exist.
\end{enumerate}

The proof for part 1 in classic MAPF and MAPF\textsubscript{R} has been shown repeatedly for CBS~\cite{CBS,li2019large,walker2020generalized}. Significantly, part 2 has been shown to be a problem for CBS and other decoupled algorithms~\cite{CBS,walker2022dissertation}. In fact, these algorithms will run forever given an unsolvable problem instance if another complete algorithm is not run in parallel. Even a simple unsolvable instance like the one shown above will cause CBS to run forever. 

Because agents can take actions such as waiting or moving back and forth without coming into conflict, the high-level search tree grows infinitely. The reason for this infinite growth is related to the inclusion of time in the state space. Time is typically part of the state space in MAPF\textsubscript{R} and decoupled algorithms, but not necessarily coupled classic MAPF algorithms. Adding the time dimension to a finite map like the one in Figure \ref{fig:instance} makes the state space infinite since each time step is distinct and agents can wait in place forever. Special attention to duplicate detection is necessary to make the search space finite.

We again emphasize that no complete, polynomial-time solvers for MAPF\textsubscript{R} are known at this time. This means that the strategy of running a separate polynomial-time algorithm in parallel to determine solvability of a MAPF\textsubscript{R} instance is not currently an option. Figure \ref{fig:mapfrinstance} illustrates a phenomenon that is introduced by MAPF\textsubscript{R}, that agents can conflict in ways that are undefined in Classic MAPF. Although the classic MAPF instance with point agents in Figure \ref{fig:mapfrinstance}(a) is solvable, the MAPF\textsubscript{R} instance in Figure \ref{fig:mapfrinstance}(b) is not. Because of this phenomenon (and others related to non-planar graphs and non-unit costs), one cannot simply analyze the graph of a MAPF\textsubscript{R} instance, nor run a classic polynomial-time MAPF solver on the graph to determine solvability.

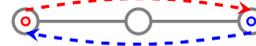
\begin{figure}[t!]

\centering
\begin{tikzpicture}[scale=1.5]

\path (.5,.5) edge [-,very thick,color=gray] node {} (2.5,.5);
\node [circle,very thick,draw=gray,fill=white,minimum size=5pt] at (.5,.5) {};
\node [circle,very thick,draw=gray,fill=white,minimum size=5pt] at (1.5,.5) {};
\node [circle,very thick,draw=gray,fill=white,minimum size=5pt] at (2.5,.5) {};

\node[circle,draw=red,text=red,thick,fill=red!10,inner sep=0pt,minimum size=3pt] at (.5,.5) {};
\node[circle,draw=blue,text=blue,thick,fill=blue!10,inner sep=0pt,minimum size=3pt] at (2.5,.5) {};

\path (.5,.6) edge [->,>=stealth,very thick,bend left=10,dashed,color=red] node {} (2.5,.6);
\path (2.5,.4) edge [->,>=stealth,very thick,bend left=10,dashed,color=blue] node {} (.5,.4);
\end{tikzpicture}
\caption{An unsolvable MAPF instance where the agents must swap places.}
\label{fig:instance}
\end{figure}

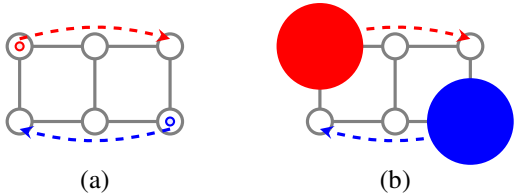
\begin{figure}[b!]

\centering
\begin{tikzpicture}[scale=1]
\node[circle,draw=white,text=white,thick,fill=white,inner sep=0pt,minimum size=32pt, fill opacity=0.2] at (3.25,.5) {};
\path (.5,.5) edge [-,very thick,color=gray] node {} (2.5,.5);
\path (.5,1.5) edge [-,very thick,color=gray] node {} (2.5,1.5);
\path (.5,.5) edge [-,very thick,color=gray] node {} (.5,1.5);
\path (1.5,.5) edge [-,very thick,color=gray] node {} (1.5,1.5);
\path (2.5,.5) edge [-,very thick,color=gray] node {} (2.5,1.5);
\node [circle,very thick,draw=gray,fill=white,minimum size=5pt] at (.5,.5) {};
\node [circle,very thick,draw=gray,fill=white,minimum size=5pt] at (.5,1.5) {};
\node [circle,very thick,draw=gray,fill=white,minimum size=5pt] at (1.5,.5) {};
\node [circle,very thick,draw=gray,fill=white,minimum size=5pt] at (1.5,1.5) {};
\node [circle,very thick,draw=gray,fill=white,minimum size=5pt] at (2.5,.5) {};
\node [circle,very thick,draw=gray,fill=white,minimum size=5pt] at (2.5,1.5) {};

\node[circle,draw=red,text=red,thick,fill=red!10,inner sep=0pt,minimum size=3pt] at (.5,1.5) {};
\node[circle,draw=blue,text=blue,thick,fill=blue!10,inner sep=0pt,minimum size=3pt] at (2.5,.5) {};

\path (.5,1.6) edge [->,>=stealth,very thick,bend left=15,dashed,color=red] node {} (2.5,1.6);
\path (2.5,.4) edge [->,>=stealth,very thick,bend left=15,dashed,color=blue] node {} (.5,.4);
\node at (1.5,-.3) {(a)};
\end{tikzpicture}
\begin{tikzpicture}[scale=1]
\path (.5,.5) edge [-,very thick,color=gray] node {} (2.5,.5);
\path (.5,1.5) edge [-,very thick,color=gray] node {} (2.5,1.5);
\path (.5,.5) edge [-,very thick,color=gray] node {} (.5,1.5);
\path (1.5,.5) edge [-,very thick,color=gray] node {} (1.5,1.5);
\path (2.5,.5) edge [-,very thick,color=gray] node {} (2.5,1.5);
\node [circle,very thick,draw=gray,fill=white,minimum size=5pt] at (.5,.5) {};
\node [circle,very thick,draw=gray,fill=white,minimum size=5pt] at (.5,1.5) {};
\node [circle,very thick,draw=gray,fill=white,minimum size=5pt] at (1.5,.5) {};
\node [circle,very thick,draw=gray,fill=white,minimum size=5pt] at (1.5,1.5) {};
\node [circle,very thick,draw=gray,fill=white,minimum size=5pt] at (2.5,.5) {};
\node [circle,very thick,draw=gray,fill=white,minimum size=5pt] at (2.5,1.5) {};

\node[circle,draw=red,text=red,thick,fill=red,inner sep=0pt,minimum size=32pt, fill opacity=0.1] at (.5,1.5) {};
\node[circle,draw=blue,text=blue,thick,fill=blue,inner sep=0pt,minimum size=32pt, fill opacity=0.1] at (2.5,.5) {};

\path (.5,1.6) edge [->,>=stealth,very thick,bend left=15,dashed,color=red] node {} (2.5,1.6);
\path (2.5,.4) edge [->,>=stealth,very thick,bend left=15,dashed,color=blue] node {} (.5,.4);
\node at (1.5,-.3) {(b)};
\end{tikzpicture}
\caption{A MAPF instance that becomes unsolvable in MAPF\textsubscript{R} with larger agents.}
\label{fig:mapfrinstance}
\end{figure}

\section{Completeness of CBS}

Before discussing the completeness of CBS, it is worth mentioning that in popular MAPF benchmark sets~\cite{stern2019multi} there are no unsolvable problem instances. Generally speaking, unsolvable instances are often hand-crafted, and (depending on the domain) tend to be very rare. Our purpose here is not to make CBS scale to larger problem instances, but to close the question of completeness by adding functionality to CBS so that it is \textit{natively} complete for classic MAPF, and identify an approach to completeness adequate for MAPF\textsubscript{R} as well.

Aside from duplicate pruning, CBS could be made complete by first computing the theoretical upper bound on the size of the CT for a problem instance, and then forcing CBS to terminate once the CT grows beyond that upper bound.

The total number of unique states in all possible paths of cost $C$ for a single agent in a Classic MAPF instance can not exceed the lesser of $C^3$ and $|V|C$, where $C$ is the makespan of the lowest-cost solution~\cite{gordon2021revisiting}.
Hence the total number of unique multi-agent states in all possible solutions is bounded by the lesser of ${C^3}^k$ and ${(|V|C})^k$.
A similar bound for MAPF\textsubscript{R} does not exist because the number of unique times for states in a path is infinite.
However, for MAPF\textsubscript{Q}, the bound is ${(Cr)^3}^k$ or $(|V|Cr)^k$. Where $r{\in}\mathbb{Q}_+$ is the inverse resolution of time (e.g., $r{=}100$ for a resolution of $\nicefrac{1}{100}$). Note, $r{=}1$ for Classic MAPF.

It has been shown that the makespan, $C$, for a solvable MAPF instance is $O(|V|^3)$~\cite{kornhauser1984coordinating}. Substituting for $C$, we find that CBS can immediately terminate when the CT reaches a size of $2^{k(|V|r)^4}$. This approach trivially makes CBS complete. However, this simple  approach is not practical. For the example instance in Figure \ref{fig:instance}, where the red agent must swap places with the blue agent, CBS could not terminate safely until generating a CT of size $O(2^{2\cdot 3^4})$. 

This is often worse than running a coupled algorithm like A*, which (with proper duplicate detection) would terminate after $O({|V|^4}^k)$ expansions. This is because a coupled algorithm in Classic MAPF doesn't have to keep track of time, just the relative locations of the agents. If A* were to include time in the duplicate detection (as it usually must for continuous-time MAPF), it too would be incomplete.  Neither of these bounds (for A* and CBS) are tight because most of the multi-agent state space is not reachable. For the example instance in Figure \ref{fig:instance}, only 3 unique states would be expanded in total by A* before termination.

The number of high-level expansions with TRDP is finite, but still has the same upper bound for CBS (on a solvable problem instance) without TRDP. However, with TRDP, the actual search space may theoretically be reduced exponentially~\cite{taylor1997pruning}, though one may not see such a reduction in practice. Running CBS with TRDP on the example instance in Figure \ref{fig:instance} terminates after only 5 high-level expansions.

\section{Mechanics of Duplicate Detection}

Duplicate pruning is a technique commonly used with search algorithms in graphs with cycles in order to eliminate exploring sub-optimal paths with loops~\cite{taylor1997pruning}. A \emph{loop} happens when a path visits the same state twice. It is trivial to see that any path with a loop is sub-optimal because concatenating the portions of the path before and after the loop results in a shorter path to the goal.

In MAPF, single-agent loops may be necessary -- a feasible solution may correctly contain many agents visiting the same vertex more than once (for example when one agent waits for another), however, \emph{$k$-agent loops} can be a source of incompleteness. For CBS and many other algorithms that include the time dimension as a part of the state, loops are not recognized by comparing states since revisiting a state (in a multi-agent sense) happens at a different time.

\emph{Temporally-relative duplicate pruning} (TRDP) defines \emph{temporally-relative duplicate} (TRD) detection which allows the removal of states which have been visited before in a temporally-relative sense.

We define TRD formally as follows: Let $S{=}\{s_1,..,s_k\}$ be a \emph{joint-state} composed of $k$ single-agent states. All single-agent states are not required to to have identical times in $S$, but in our definition, the respective actions taken to arrive at $S$, must have time overlap. Specifically, for each single-agent action $(\hat{s_i}, s_i)$ for $\hat{s_i}{\in}\hat{S}$ and $s_i{\in}S$ where $\hat{S}$ is the ancestor of $S$, must have time overlap with every other action $(\hat{s_j}, s_j)$:
\begin{equation*}
  \begin{aligned}
  \hat{s_i}.t \leq s_j.t \leq s_i.t \\
\text{or\;\;\;\;\;\;\;\;\;\;\,}\\
\hat{s_j}.t \leq s_i.t \leq s_j.t
\end{aligned}
\end{equation*}

To define a TRD $S'$ of $S$, we first define $t_\text{min}(S)$ to be the earliest single-agent time of $S$:
$$t_\text{min}(S)=\underset{s_i\in S}{\textsc{min}} s_i.t$$

Next, we define a joint-state temporal adjustment function ${\Delta_t}(S)$, which adjusts the time component of all single-agent states in $S$ to be relative to $t_\text{min}(S)$:
$${\Delta_t}(S)=\{\forall s_i{\in}S; (s_i.v, s_i.t-t_\text{min}(S))\}$$

$S'$, a descendant of $S$, is a TRD of $S$ iff the vertex part of the states are identical and the temporally-relative times are identical. That is: ${\Delta_t}(S)={\Delta_t}(S')$. For example, from Figure \ref{fig:instance} we could have a joint state $S=\{(A1,0.1), (A3,0.2)\}$. Thus, $t_\text{min}(S)=0.1$, and therefore ${\Delta_t}(S)=\{(A1,0.0), (A3,0.1)\}$. Supposing another state $S'=\{(A1,0.2),(A3,0.3)\}$, then ${\Delta_t}(S')=\{(A1,0.0), (A3,0.1)\}$. Because ${\Delta_t}(S)={\Delta_t}(S')$, $S'$ is a TRD of $S$.

Pruning duplicates from the search space renders a search space finite under certain assumptions~\cite{taylor1997pruning}. However, in MAPF, it is important to ensure that the duplicates are temporally relative. Continuing the previous example, let $S''=\{(A1,0.1),(A3,0.3)\}$. $S''$ is not a TRD of $S$; although the agents are at the same locations $(A1,A3)$ respectively, they are not at those locations at the same \emph{relative} time; ${\Delta_t}(S){\neq}{\Delta_t}(S'')$.

If $S''$ were part of the only feasible solution $\Pi^*$, and $S''$ were to be pruned (as an erroneous TRD of $S$), then the algorithm would never find $\Pi^*$, making the algorithm incomplete. On the other hand, if $S'$ (being a proper TRD of $S$) were to be pruned, finding $\Pi^*$ is not precluded because any path to the goal from $S'$ is identical (in a temporally relative sense) to a path to the goal from $S$.

\section{Duplicate Pruning in CBS}

Our discussion will now focus on the implementation of TRDP in CBS~\cite{CBS}. First, we explain the CBS algorithm, then we explain the TRDP implementation. Pseudocode for the CBS algorithm is shown in Algorithm \ref{alg:trdp} with changes for TRDP on line \ref{alg1:trdcheck}.

\subsection{CBS}
CBS searches a conflict tree (CT) where each node $N$ contains a solution $N.\Pi$. The root node contains paths for each agent without taking the paths of other agents into account (line \ref{alg1:root}). $N.\Pi$ may contain conflicts, and CBS will detect conflicts between the paths in $N.\Pi$ (line \ref{alg1:concheck}). If any conflict exists between the paths for agents $i$ and $j$, two child nodes $N_i$ and $N_j$ are generated, one for each agent in conflict. Each child node contains a new constraint for a single agent to avoid the conflict (line \ref{alg1:constraint}).

Many different types of constraints are possible. For example, a vertex constraint $c = \langle i,v,t\rangle$ blocks agent $i$ from entering vertex $v$ at time $t$. Taking constraints for $N_i$ and all its ancestors into account (line \ref{alg1:ancestor}), $\pi_i{\in}N_i.\Pi$ is re-planned (line \ref{alg1:replan}) to respect the new constraints (and consequently avoid the conflict with agent $j$). This process is done analogously for agent $j$. In any case that the agent cannot reach its goal in the context of all constraints (i.e., over-constrained) the CT node is pruned (line \ref{alg1:prune}). Eventually, after enough constraints are accumulated, CBS will find a feasible solution iff one exists.

A partial example of a CT is shown in Figure \ref{fig:incomplete}. The root node contains shortest paths to the goal and child nodes resolve conflicts by adding constraints to one agent or the other, which in this case causes the agents to wait. In the figure, an arrow ($\leftarrow$) means an action that moves the agent to an adjacent node, and a looped arrow ($\wait$) means a wait action. A wait action at the goal is shown with a box around it. The pruning shown with `X's in the figure is due to TRDP and will be discussed later. The search stops when a conflict-free solution is found in the node with a green border.


\subsection{CBS+TRDP}

As previously mentioned, any solution containing a $k$-agent loop is sub-optimal. Since CBS is an optimal algorithm, a solution containing a $k$-agent loop in a CT node can be treated similar to an infeasible solution. Therefore, TRDP introduces a new type of conflict called a temporally-relative duplicate conflict or \textit{TRD conflict}. A TRD conflict occurs when a temporally-relative $k$-agent loop exists in a solution $N.\Pi$. Thus, in addition to detecting traditional motion conflicts between pairs of agents (see Algorithm \ref{alg:trdp} line \ref{alg1:concheck}), CBS+TRDP detects TRD conflicts (Algorithm \ref{alg:trdp} line \ref{alg1:trdcheck}) using the approach defined in the previous section (Algorithm \ref{alg:trdd}).

In response to a TRD conflict, TRDP generates $k$ constraints, one for each agent, which causes CBS to generate $k$ child nodes $N_1,...,N_k$, such that each child node $N_i$ contains a TRD constraint. A \emph{TRD constraint} is a tuple $\langle i, [t_\text{start}, t_\text{end}), t_\text{offset} \rangle$, where $i$ is the agent number, $[t_\text{start}, t_\text{end})$ is a time range in which states are considered to be a \textit{loop-start candidate} state, and $t_\text{offset}$ is an exact time offset from a loop-start candidate. For example, a TRD constraint of: $\langle i, [0.1,0.3), 0.4\rangle$, means that for agent $i$, any state in the time range $[0.1,0.3)$ will be considered a loop-start candidate, and any descendant state at the same vertex with an exact time offset of $0.4$ is considered a duplicate.

Thus, a TRD constraint enforces the rule that the agent being re-planned at the low level will avoid \textit{all} states $s'$, which are duplicates of \textit{any} loop-start candidate $s$, such that $s.t\in [t_\text{start}, t_\text{end})$, $s.v{=}s'.v$ and $s'.t{=}s.t+t_\text{offset}$. For example, during low-level planning, given a TRD constraint of: $\langle i, [0.1,0.3), 0.4\rangle$, if the agent's path were to visit the loop-start candidate $s{=}(v, 0.2)$, it would not be allowed to visit $s'{=}(v, 0.6)$ because it visits the same vertex at the exact offset time.

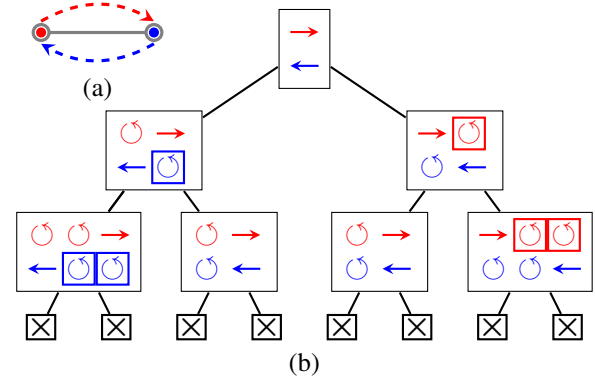
\begin{figure}[b!]

\centering
\begin{tikzpicture}[scale=1, node distance=2.4cm, thick,
cloud/.style={inner sep=0pt,fill=none}]

\node[cloud] (q) at (-2.75,0) {
    \begin{tikzpicture}[scale=1.5,anchor=center]
\path (.5,.5) edge [-,very thick,color=gray] node[text=black,shift={(0,.16)}] {} (1.5,.5);
\node [circle,very thick,draw=gray,fill=white,minimum size=6pt] at (.5,.5) {};
\node [circle,very thick,draw=gray,fill=white,minimum size=6pt] at (1.5,.5) {};
\node[circle,draw=blue,text=blue,thick,fill=blue,inner sep=0pt,minimum size=3pt, fill opacity=0.1] at (1.5,.5) {};
\node[circle,draw=red,text=red,thick,fill=red,inner sep=0pt,minimum size=3pt, fill opacity=0.1] at (.5,.5) {};

\path (.5,.6) edge [->,>=stealth,very thick,bend left=30,dashed,color=red] node {} (1.5,.6);
\path (1.5,.4) edge [->,>=stealth,very thick,bend left=30,dashed,color=blue] node {} (.5,.4);

\node[] at (1,0) {(a)};
    \end{tikzpicture}
};

\node[cloud] (a) {
    \begin{tikzpicture}[scale=1.5,anchor=center]
    \draw[color=black] (.15,0) rectangle (.6,.7);
    \path (.25,.5) edge [->,>=stealth,thick,color=red] node {} (.5,.5);
    \path (.5,.2) edge [->,>=stealth,thick,color=blue] node {} (.25,.2);
    \end{tikzpicture}
};

\node[cloud,below of=a] at ([shift={(-2,1.05)}]a) (aa){
    \begin{tikzpicture}[scale=1.5,anchor=center]
    \draw[color=black] (.15,0) rectangle (1,.7);
    \node[circle,draw=white,text=red,thick,solid,fill=white,inner sep=0pt,minimum size=15pt] at (.375,.5) {\Large $\wait$};
    \node[draw=blue,text=blue,thick,solid,fill=white,inner sep=0pt,minimum size=12pt] at (.7,.2) {\Large $\wait$};
    \path (.6,.5) edge [->,>=stealth,thick,color=red] node {} (.85,.5);
    \path (.5,.2) edge [->,>=stealth,thick,color=blue] node {} (.25,.2);
    \end{tikzpicture}
};

\node[cloud,below of=a] at ([shift={(2,1.05)}]a) (ab) {
    \begin{tikzpicture}[scale=1.5,anchor=center]
    \draw[color=black] (.15,0) rectangle (1,.7);
    \node[draw=red,text=red,thick,solid,fill=white,inner sep=0pt,minimum size=12pt] at (.7,.5) {\Large $\wait$};
    \node[draw=white,text=blue,thick,solid,fill=white,inner sep=0pt,minimum size=1pt] at (.375,.2) {\Large $\wait$};
    \path (.25,.5) edge [->,>=stealth,thick,color=red] node {} (.5,.5);
    \path (.85,.2) edge [->,>=stealth,thick,color=blue] node {} (.6,.2);
    \end{tikzpicture}
};

\node[cloud,below of=aa] at ([shift={(-1,1.05)}]aa) (aaa){
    \begin{tikzpicture}[scale=1.5,anchor=center]
    \draw[color=black] (.15,0) rectangle (1.25,.7);
    \node[circle,draw=white,text=red,thick,solid,fill=white,inner sep=0pt,minimum size=15pt] at (.375,.5) {\Large $\wait$};
    \node[circle,draw=white,text=red,thick,solid,fill=white,inner sep=0pt,minimum size=15pt] at (.7,.5) {\Large $\wait$};
    \node[draw=blue,text=blue,thick,solid,fill=white,inner sep=0pt,minimum size=12pt] at (.7,.2) {\Large $\wait$};
    \node[draw=blue,text=blue,thick,solid,fill=white,inner sep=0pt,minimum size=12pt] at (1,.2) {\Large $\wait$};
    \path (.9,.5) edge [->,>=stealth,thick,color=red] node {} (1.15,.5);
    \path (.5,.2) edge [->,>=stealth,thick,color=blue] node {} (.25,.2);
    \end{tikzpicture}
};

\node[cloud,below of=aa] at ([shift={(1,1.05)}]aa) (aab){
    \begin{tikzpicture}[scale=1.5,anchor=center]
    \draw[color=black] (.15,0) rectangle (1,.7);
    \node[circle,draw=white,text=red,thick,solid,fill=white,inner sep=0pt,minimum size=15pt] at (.375,.5) {\Large $\wait$};
    \path (.6,.5) edge [->,>=stealth,thick,color=red] node {} (.85,.5);
    \node[circle,draw=white,text=blue,thick,solid,fill=white,inner sep=0pt,minimum size=15pt] at (.375,.2) {\Large $\wait$};
    \path (.85,.2) edge [->,>=stealth,thick,color=blue] node {} (.6,.2);
    \end{tikzpicture}
};

\node[cloud,below of=ab] at ([shift={(-1,1.05)}]ab) (aba){
    \begin{tikzpicture}[scale=1.5,anchor=center]
    \draw[color=black] (.15,0) rectangle (1,.7);
    \node[circle,draw=white,text=red,thick,solid,fill=white,inner sep=0pt,minimum size=15pt] at (.375,.5) {\Large $\wait$};
    \path (.6,.5) edge [->,>=stealth,thick,color=red] node {} (.85,.5);
    \node[circle,draw=white,text=blue,thick,solid,fill=white,inner sep=0pt,minimum size=15pt] at (.375,.2) {\Large $\wait$};
    \path (.85,.2) edge [->,>=stealth,thick,color=blue] node {} (.6,.2);
    \end{tikzpicture}
};

\node[cloud,below of=ab] at ([shift={(1,1.05)}]ab) (abb){
    \begin{tikzpicture}[scale=1.5,anchor=center]
    \draw[color=black] (.15,0) rectangle (1.25,.7);

    \node[circle,draw=white,text=blue,thick,solid,fill=white,inner sep=0pt,minimum size=15pt] at (.375,.2) {\Large $\wait$};
    \node[circle,draw=white,text=blue,thick,solid,fill=white,inner sep=0pt,minimum size=15pt] at (.7,.2) {\Large $\wait$};
    \path (1.15,.2) edge [->,>=stealth,thick,color=blue] node {} (.9,.2);
    \node[draw=red,text=red,thick,solid,fill=white,inner sep=0pt,minimum size=12pt] at (.7,.5) {\Large $\wait$};
    \node[draw=red,text=red,thick,solid,fill=white,inner sep=0pt,minimum size=12pt] at (1,.5) {\Large $\wait$};
    \path (.25,.5) edge [->,>=stealth,thick,color=red] node {} (.5,.5);
    \end{tikzpicture}
};

\node[below of=aaa,draw=black,inner sep=0] at ([shift={(-.5,1.4)}]aaa) (aaaa){\Large $\boldsymbol{\times}$};
\node[below of=aaa,draw=black,inner sep=0] at ([shift={(.5,1.4)}]aaa) (aaab){\Large $\boldsymbol{\times}$};
\node[below of=aab,draw=black,inner sep=0] at ([shift={(-.5,1.4)}]aab) (aaba){\Large $\boldsymbol{\times}$};
\node[below of=aab,draw=black,inner sep=0] at ([shift={(.5,1.4)}]aab) (aabb){\Large $\boldsymbol{\times}$};
\node[below of=aba,draw=black,inner sep=0] at ([shift={(-.5,1.4)}]aba) (abaa){\Large $\boldsymbol{\times}$};
\node[below of=aba,draw=black,inner sep=0] at ([shift={(.5,1.4)}]aba) (abab){\Large $\boldsymbol{\times}$};
\node[below of=abb,draw=black,inner sep=0] at ([shift={(-.5,1.4)}]abb) (abba){\Large $\boldsymbol{\times}$};
\node[below of=abb,draw=black,inner sep=0] at ([shift={(.5,1.4)}]abb) (abbb){\Large $\boldsymbol{\times}$};

\path (a) edge [-] (aa)
      (a) edge [-] (ab)
      (aa) edge [-] (aaa)
      (aa) edge [-] (aab)
      (aa) edge [-] (aaa)
      (ab) edge [-] (abb)
      (ab) edge [-] (aba)
      (aaa) edge [-] (aaaa)
      (aaa) edge [-] (aaab)
      (aab) edge [-] (aaba)
      (aab) edge [-] (aabb)
      (aba) edge [-] (abaa)
      (aba) edge [-] (abab)
      (abb) edge [-] (abba)
      (abb) edge [-] (abbb)
      ;

\node[] at (0,-4.2) {(b)};

\end{tikzpicture}

\caption{(a) An unsolvable MAPF instance and (b) a partial CBS high-level search tree with TRDP.}
\label{fig:complete}
\end{figure}

\begin{algorithm}[t!]
\caption{CBS+TRDP Algorithm}
\label{alg:trdp}
{\fontsize{8}{8}\selectfont
\begin{algorithmic}[1]
\algrenewcommand\algorithmicindent{1em}%
\State Input: a MAPF instance
\State $OPEN\gets\emptyset$
\State Initialize the root node and add it to $OPEN$ \label{alg1:root}
\While{$OPEN\neq\emptyset$}
\State $N\gets OPEN.pop()$
\State $\color{red}C\gets\textsc{findTRDs}(N.\Pi)$  \label{alg1:trdcheck} \Comment Find TRDs, return constraints
\If{$C=\emptyset$}
  \State $C\gets\textsc{findConflict}(N.\Pi)$  \label{alg1:concheck} \Comment Find conflicts, return constraints
\EndIf
\If{$C=\emptyset$}
  \State \textbf{return} $N.\Pi$
\Else
  \For{$i,c \in C$} \Comment For each agent $i$ and set of constraints $c$ \label{alg1:foreach}
      \State $N'\gets N$\label{alg1:ancestor}
      \State $N'.C\gets N'.C\cup c$ \label{alg1:constraint}
      \State $\pi' \gets Replan(i,N'.C)$ \label{alg1:replan}
      \If{$\pi' \neq \emptyset$} \label{alg1:prune}
        \State $N'.\Pi.\pi_i\gets\pi'$
        \State $OPEN\gets OPEN \cup N'$
      \EndIf
  \EndFor \label{alg1:conjchildend}
\EndIf
\EndWhile
\State \textbf{return} "No solution"
\end{algorithmic}
}
\end{algorithm}

\begin{algorithm}[t!]
\caption{\textsc{findTRDs}: TRD Detection and Constraint Generation}
\label{alg:trdd}
{\fontsize{8}{8}\selectfont
\begin{algorithmic}[1]
\algrenewcommand\algorithmicindent{1em}%
\State Input: a MAPF solution $\Pi$.
\State Check if a loop $\langle s,s'\rangle$ occurs in any $\pi_i\in\Pi$; save pointers to loops. \label{alg2:check1}
\State If fewer than $k$ paths have loops, return $\emptyset$.
\State Check for TRDs $\langle S,S' \rangle$ using pointers to loops from step \ref{alg2:check1}. \label{alg2:check2}
\State If no $k$-agent loop exists, return $\emptyset$.
\State Create $k$ TRD constraints for each agent: \label{alg1:kcstr}
\For {$i\in 1..k$ }
\State $c_i\gets\langle i, [t_\text{min}(S),t_\text{max}(S)), t_\text{min}(S')-t_\text{min}(S) \rangle$
\EndFor
\State \textbf{return} $\{c_1,...,c_k\}$, the set of TRD constraints, one for each agent.

\end{algorithmic}
}
\end{algorithm}

The expansion routine for the low-level A* search is shown in Algorithm \ref{alg:trdpex}. During the expansion, for any loop-start candidate, we add a \emph{TRD list} to search nodes that contains a list of \emph{TRD pairs} (line \ref{alg3:pairs}). Each TRD pair contains a pointer to an ancestor node, $s$, and an offset time $d$.offset. Once set, the TRD list is propagated to descendants.

The regular CBS constraint pruning check is carried out first (line \ref{alg3:constr}). Then the TRD constraint pruning check is carried out (line \ref{alg3:trdp}). In this check, if the vertex is the same as some loop-start candidate, and the offset time criteria is exactly met, then the node $s'$ is not generated. Note that $s'$ could still be generated if reached by another path through an ancestor that is not a loop-start candidate.

A* normally contains a dominance check which replaces an open node with any node of lower cost. This is still part of our expansion routine. However if a successor is visited via another path with equal g-cost, but the node has no TRDlist, the TRD list is updated (line \ref{alg3:dom2}). This effectively removes the TRDlist and ensures that if any equally low-cost path exists which avoids the loop, it will be retained.

\begin{algorithm}[t!]
\caption{Low Level A* Expansion Routine}
\label{alg:trdpex}
{\fontsize{8}{8}\selectfont
\begin{algorithmic}[1]
\algrenewcommand\algorithmicindent{1em}%
\State Input: $n$ a single-agent search node,
\State $\;\;\;\;\;\;\;\;\,C$ a set of "regular" CBS constraints,
\State $\;\;\;\;\;\;\;\;\,D$ a set of TRD constraints
\State $\triangleright$ Generate successor nodes.
\For{$s' \in \textsc{successors}(n.s)$} \label{algexp:succloop}
  \State $\triangleright$ Prune successors that violate regular CBS constraints.
  \If{\textit{any} $c \in C$ \textit{blocks} $s'$} \label{alg3:constr}
    \State Back to top of loop on line \ref{algexp:succloop}. \Comment{Successor is pruned.}
  \EndIf
  \State \color{red}$\triangleright$ Prune successors that would complete a multi-agent loop.
  \For{TRDpair $\in n.\text{TRDlist}$} \label{alg3:trdp}
    \If{TRDpair$.s.v = s'.v$ AND \\$\;\;\;\;\;\;\;\;\;\;\;\text{TRDpair}.s.t + \text{TRDpair}.\text{offset} = s'.t$}
      \State Back to top of loop on line \ref{algexp:succloop} \Comment{Successor is pruned.}
    \EndIf
  \EndFor
  \State $n' \gets \{s',\text{n.gcost}+cost(s,s'),n.\text{TRDlist}\}$ \Comment{Successor node.}
  \State $\triangleright$ Add TRDpairs if applicable. \label{alg3:pairs}
  \For{$d \in D$}
    \If{$s'.t >= d.t_\text{start}$ AND $s'.t <= d.t_\text{end}$}
      \State $n'.\text{TRDlist} \gets n'.\text{TRDlist} \cup \langle s', d.\text{offset}\rangle$ \Comment{Add TRDpair.}
    \EndIf
  \EndFor\color{black}
  \State $\triangleright$ Check for dominance.
  \If{$n' \in OPEN$ AND \\
  \;\;($n'.\text{gcost} < OPEN.\text{fetch}(n').\text{gcost}$ \color{red}OR \\ \label{alg:dom1}
  $\;\;(n'.\text{gcost} = OPEN.\text{fetch}(n').\text{gcost}$ AND $n'.\text{TRDlist} = \emptyset$)\color{black})} \label{alg3:dom2}
    \State $OPEN.\text{fetch}(n') \gets n'$ \Comment{Update OPEN node.}
  \Else
    \State $OPEN \gets OPEN \cup n'$
  \EndIf
\EndFor
\end{algorithmic}
}
\end{algorithm}

\begin{figure}[b!]

\centering
\begin{tikzpicture}[scale=1, node distance=2.4cm, thick,
cloud/.style={inner sep=0pt,fill=none}]

\node[cloud] (q) at (-3.4,-0.3) {
    \begin{tikzpicture}[scale=1.5,anchor=center]
\path (.5,.5) edge [-,very thick,color=gray] node[text=black,shift={(0,.16)}] {1} (1.5,.5);
\path (.5,.45) edge [->,>=stealth,very thick,bend right=85,color=gray, looseness=1.5] node[text=black,shift={(0,.15)}] {100} (1.5,.45);
\node [circle,very thick,draw=gray,fill=white,minimum size=6pt] at (.5,.5) {};
\node [circle,very thick,draw=gray,fill=white,minimum size=6pt] at (1.5,.5) {};
\node[circle,draw=blue,text=blue,thick,fill=blue,inner sep=0pt,minimum size=3pt, fill opacity=0.1] at (1.5,.5) {};
\node[circle,draw=red,text=red,thick,fill=red,inner sep=0pt,minimum size=3pt, fill opacity=0.1] at (.5,.5) {};

\path (.5,.6) edge [->,>=stealth,very thick,bend left=30,dashed,color=red] node {} (1.5,.6);
\path (1.5,.4) edge [->,>=stealth,very thick,bend left=30,dashed,color=blue] node {} (.5,.4);

\node[] at (1,-.3) {(a)};
    \end{tikzpicture}
};

\node[cloud] (a) {
    \begin{tikzpicture}[scale=1.5,anchor=center]
    \draw[color=black] (.15,0) rectangle (.6,.7);
    \path (.25,.5) edge [->,>=stealth,thick,color=red] node {} (.5,.5);
    \path (.5,.2) edge [->,>=stealth,thick,color=blue] node {} (.25,.2);
    \end{tikzpicture}
};

\node[cloud,below of=a] at ([shift={(-1.8,1.05)}]a) (aa){
    \begin{tikzpicture}[scale=1.5,anchor=center]
    \draw[color=black] (.15,0) rectangle (1,.7);
    \node[circle,draw=white,text=red,thick,solid,fill=white,inner sep=0pt,minimum size=15pt] at (.375,.5) {\Large $\wait$};
    \node[draw=blue,text=blue,thick,solid,fill=white,inner sep=0pt,minimum size=12pt] at (.7,.2) {\Large $\wait$};
    \path (.6,.5) edge [->,>=stealth,thick,color=red] node {} (.85,.5);
    \path (.5,.2) edge [->,>=stealth,thick,color=blue] node {} (.25,.2);
    \end{tikzpicture}
};

\node[cloud,below of=a] at ([shift={(1.8,1.05)}]a) (ab) {
    \begin{tikzpicture}[scale=1.5,anchor=center]
    \draw[color=black] (.15,0) rectangle (1,.7);
    \node[draw=red,text=red,thick,solid,fill=white,inner sep=0pt,minimum size=12pt] at (.7,.5) {\Large $\wait$};
    \node[circle,draw=white,text=blue,thick,solid,fill=white,inner sep=0pt,minimum size=15pt] at (.375,.2) {\Large $\wait$};
    \path (.25,.5) edge [->,>=stealth,thick,color=red] node {} (.5,.5);
    \path (.85,.2) edge [->,>=stealth,thick,color=blue] node {} (.6,.2);
    \end{tikzpicture}
};

\node[cloud,below of=aa] at ([shift={(-.8,1.05)}]aa) (aaa){
    \begin{tikzpicture}[scale=1.5,anchor=center]
    \draw[color=black] (.15,0) rectangle (1.25,.7);
    \node[circle,draw=white,text=red,thick,solid,fill=white,inner sep=0pt,minimum size=15pt] at (.375,.5) {\Large $\wait$};
    \node[circle,draw=white,text=red,thick,solid,fill=white,inner sep=0pt,minimum size=15pt] at (.7,.5) {\Large $\wait$};
    \node[draw=blue,text=blue,thick,solid,fill=white,inner sep=0pt,minimum size=12pt] at (.7,.2) {\Large $\wait$};
    \node[draw=blue,text=blue,thick,solid,fill=white,inner sep=0pt,minimum size=12pt] at (1,.2) {\Large $\wait$};
    \path (.9,.5) edge [->,>=stealth,thick,color=red] node {} (1.15,.5);
    \path (.5,.2) edge [->,>=stealth,thick,color=blue] node {} (.25,.2);
    \end{tikzpicture}
};

\node[cloud,below of=aaa] at ([shift={(-.8,1.05)}]aaa) (aaaa){
    \begin{tikzpicture}[scale=1.5,anchor=center]
    \draw[color=green!50!black,very thick] (.15,0) rectangle (1.25,.7);
    \node[draw=blue,text=blue,thick,solid,fill=white,inner sep=0pt,minimum size=12pt] at (1,.2) { $\ldots$};
    \node[draw=blue,text=blue,thick,solid,fill=white,inner sep=0pt,minimum size=12pt] at (.7,.2) {\Large $\wait$};

    \path (.5,.2) edge [->,>=stealth,thick,color=blue] node {} (.25,.2);
    
    \path (.275,.5) edge [->,>=stealth,thick,bend right=25,color=red, looseness=1] node[text=red,shift={(0,.18)}] {100} (1.1,.5);
    \end{tikzpicture}
};

\node[cloud,below of=aa] at ([shift={(.8,1.05)}]aa) (aab){
    \begin{tikzpicture}[scale=1.5,anchor=center]
    \draw[color=black] (.15,0) rectangle (1,.7);
    \node[circle,draw=white,text=red,thick,solid,fill=white,inner sep=0pt,minimum size=15pt] at (.375,.5) {\Large $\wait$};
    \path (.6,.5) edge [->,>=stealth,thick,color=red] node {} (.85,.5);
    \node[circle,draw=white,text=blue,thick,solid,fill=white,inner sep=0pt,minimum size=15pt] at (.375,.2) {\Large $\wait$};
    \path (.85,.2) edge [->,>=stealth,thick,color=blue] node {} (.6,.2);
    \end{tikzpicture}
};

\node[cloud,below of=ab] at ([shift={(-.8,1.05)}]ab) (aba){
    \begin{tikzpicture}[scale=1.5,anchor=center]
    \draw[color=black] (.15,0) rectangle (1,.7);
    \node[circle,draw=white,text=red,thick,solid,fill=white,inner sep=0pt,minimum size=15pt] at (.375,.5) {\Large $\wait$};
    \path (.6,.5) edge [->,>=stealth,thick,color=red] node {} (.85,.5);
    \node[circle,draw=white,text=blue,thick,solid,fill=white,inner sep=0pt,minimum size=15pt] at (.375,.2) {\Large $\wait$};
    \path (.85,.2) edge [->,>=stealth,thick,color=blue] node {} (.6,.2);
    \end{tikzpicture}
};

\node[cloud,below of=ab] at ([shift={(.8,1.05)}]ab) (abb){
    \begin{tikzpicture}[scale=1.5,anchor=center]
    \draw[color=black] (.15,0) rectangle (1.25,.7);

    \node[circle,draw=white,text=blue,thick,solid,fill=white,inner sep=0pt,minimum size=15pt] at (.375,.2) {\Large $\wait$};
    \node[circle,draw=white,text=blue,thick,solid,fill=white,inner sep=0pt,minimum size=15pt] at (.7,.2) {\Large $\wait$};
    \path (1.15,.2) edge [->,>=stealth,thick,color=blue] node {} (.9,.2);
    \node[draw=red,text=red,thick,solid,fill=white,inner sep=0pt,minimum size=12pt] at (.7,.5) {\Large $\wait$};
    \node[draw=red,text=red,thick,solid,fill=white,inner sep=0pt,minimum size=12pt] at (1,.5) {\Large $\wait$};
    \path (.25,.5) edge [->,>=stealth,thick,color=red] node {} (.5,.5);
    \end{tikzpicture}
};

\node[below of=aab,draw=black,inner sep=0] at ([shift={(-.5,1.4)}]aab) (aaba){\Large $\boldsymbol{\times}$};
\node[below of=aab,draw=black,inner sep=0] at ([shift={(.5,1.4)}]aab) (aabb){\Large $\boldsymbol{\times}$};
\node[below of=aba,draw=black,inner sep=0] at ([shift={(-.5,1.4)}]aba) (abaa){\Large $\boldsymbol{\times}$};
\node[below of=aba,draw=black,inner sep=0] at ([shift={(.5,1.4)}]aba) (abab){\Large $\boldsymbol{\times}$};
\node[below of=abb,draw=black,inner sep=0] at ([shift={(-.5,1.4)}]abb) (abba){\Large $\boldsymbol{\times}$};
\node[below of=abb,draw=black,inner sep=0] at ([shift={(.5,1.4)}]abb) (abbb){\Large $\boldsymbol{\times}$};
\node[below of=aaa,draw=black,inner sep=0] at ([shift={(.5,1.4)}]aaa) (aaab){\Large $\boldsymbol{\times}$};

\path (a) edge [-] (aa)
      (a) edge [-] (ab)
      (aa) edge [-] (aaa)
      (aa) edge [-] (aab)
      (aa) edge [-] (aaa)
      (ab) edge [-] (abb)
      (ab) edge [-] (aba)
      (aab) edge [-] (aaba)
      (aab) edge [-] (aabb)
      (aba) edge [-] (abaa)
      (aba) edge [-] (abab)
      (abb) edge [-] (abba)
      (abb) edge [-] (abbb)
      (aaa) edge [-] (aaab)
      (aaa) edge [-] (aaaa)
      ;

\node[] at (0,-4.4) {(b)};

\end{tikzpicture}

\caption{(a) A solvable MAPF\textsubscript{R} instance and (b) a partial CBS high-level search tree with TRDP.}
\label{fig:incomplete}
\end{figure}

TRDP can be implemented in CBS by modifying the expansion routine as shown in Algorithm \ref{alg:trdp}. The only change is on Line \ref{alg1:trdcheck} which adds a TRD conflict check before the regular conflict check. If a TRD conflict is found, a node is generated for each agent with appropriate TRD constraints to help it avoid the $k$-agent loop.

The TRD detection and constraint generation algorithm is shown in Algorithm \ref{alg:trdd}. First, each path $\pi{\in}\Pi$ is checked for a single-agent loop (line 2). If all $k$ paths contain a loop, then all paths are checked for loops which have time overlap (line 4). Finally, TRD constraints are created for each agent with the appropriate information.

Given the MAPF instance in Figure \ref{fig:complete}(a), Figure \ref{fig:complete}(b) shows the resulting CT when TRDP is used.
In the root node, each agent's path moves straight to its goal. After the first split, the red agent is forced to wait on the first time step in the left child and the blue agent is forced to wait on the first time step in the right child. Wait actions at the goal are taken into account in the TRD check. TRD constraints are added to each of the four leaf nodes because they each contain a 2-agent loop. Because the TRD constraint in each case blocks the agent from entering the loop, and the regular constraint caused by the split blocks the agent from the conflict, the result is that the low-level solver cannot find a path to the goal and therefore the CT node is pruned. Because of this, CBS correctly terminates with $\emptyset$.

Figure \ref{fig:incomplete}(a) adds a higher-cost directed edge to the instance in \ref{fig:complete}(a) so that the instance is solvable. In Figure \ref{fig:incomplete}(b) the red agent is allowed to take the longer edge in the leftmost leaf node and achieve the goal. This goal is found quickly because moving on the shorter edge is blocked by regular constraints and waiting is blocked by the TRD constraint. Without TRDP, an exponential number of nodes would have to be generated before the cost could increase enough for CBS to accept the case of the red agent taking the longer edge as the optimal solution.

\section{Theoretical Analysis}

We now show that CBS+TRDP is optimal and complete. First, we show that $k$-agent loops are sub-optimal. Second, resolving them using TRD constraints can never block an optimal, feasible solution. Finally, that if the problem instance is unsolvable, TRDP guarantees termination.

\begin{definition}
    \textbf{Over-constrained:} \normalfont A CT node is over-constrained when a solution cannot be found in its sub-tree due to a constraint or collection of constraints which blocks a solution.
\end{definition}

\begin{definition}
    \textbf{Loop:} \normalfont A loop occurs when a single agent visits the same vertex more than once in its path. This includes waiting in place for one or more actions.
\end{definition}

\begin{lemma}
    \label{obs:sub}
    An optimal path cannot contain a loop.
\end{lemma}
\begin{proof}
 An optimal path cannot contain a loop because a shorter path can always be obtained by concatenating the sub-path before the loop to the sub-path after the loop.
\end{proof}

\begin{corollary}
\label{cor:sub}
    By extension of Lemma \ref{obs:sub}, if a solution has a temporally-relative $k$-agent loop, the solution must be sub-optimal.
\end{corollary}

Now that we have shown that TRDs are sub-optimal, we show that pruning them with TRDP never precludes an optimal solution.

\begin{lemma}
    \label{lem:elim}
    CBS+TRDP only eliminates sub-optimal or infeasible solutions.
\end{lemma}
\begin{proof}
    Let $N$ be a CT node for which $N.\Pi$ contains a $k$-agent loop. $N.\Pi$ is sub-optimal per Lemma \ref{obs:sub}. Consider all optimal solutions which can be reached on a path in the CT:

    \begin{enumerate}[noitemsep]
        \item There is no path through $N$ to an optimal solution.
        \item There is a path through $N$ to an optimal solution.
    \end{enumerate}
    
For case (1), either some agents are over-constrained in $N$, or there is no solution to the problem instance in general. Since an optimal solution is not reachable through $N$, any actions blocked by $c_i$ have no effect on finding an optimal solution in the sub-tree of $N$.

The proof for case (2) is very similar to the original proof for CBS. In this case, it is certain that at least one of the single-agent loops must be avoided in an optimal solution, thus at least one $c_i$ added to one of the child nodes $N_i$ in the $k$-way split (based on the $k$ TRD constraints generated in Algorithm \ref{alg:trdp} line \ref{alg1:trdcheck}) must be correct. If any $c_i$ is incorrect (blocks an optimal solution) then that solution is guaranteed to be found in at least one sibling node of $N_i$. By contradiction, if the optimal solution is blocked in every sub-tree, then either $N$ is already over-constrained (contradicting our assumptions) or a TRD can exist in an optimal solution, but this violates Lemma \ref{obs:sub}.

In summary, TRDP either prunes unsolvable sub-trees, or the optimal solution is guaranteed to lie in at least one sub-tree of the child nodes of $N$.
\end{proof}

Next, we show that TRDP guarantees termination.

\begin{lemma}
    \label{lem:finite}
    TRDP renders the search space of CBS finite for Classic MAPF and MAPF\textsubscript{Q}.
\end{lemma}
\begin{proof}
    Let $E$, the set of edges of $G$, be finite. Let $W{\in}\mathbb{Q}$ be the set of edge weights for $E$. $W$ must be finite. Recall that duplicate detection is performed based on a transformed form of $S$ via the function $\Delta_t(S)$. This is done by subtracting $t_{min}(S)$ from the time ($s.t$) of each single-agent state in $S$. The range of $t_{min}(S)$ consists of combinations of multiples of $W$, hence the maximum number of unique times in the range of $\Delta_t(S)$ is $O(\textsc{MAX}(W)/\textsc{GCD}(W))$, where $\textsc{GCD}$ is the greatest common denominator of floating point numbers. Because $W{\in}\mathbb{Q}$, $\textsc{GCD}(W)$ is well-defined and the range of times for $\Delta_t(S)$ is finite. Since the range of times for $\Delta_t(S)$ is finite, and $V$ is finite, the range of $\Delta_t(S)$ is finite.

    In summary, although the domain of $S$ is infinite because there is no upper bound on the domain of $s.t$, CBS+TRDP performs duplicate detection on $\Delta_t(S)$ which has a finite range, therefore the search space of CBS+TRDP is finite.
\end{proof}

Finally, we combine Lemmas \ref{lem:elim} and \ref{lem:finite} to show that CBS+TRDP is optimal and complete.

\begin{thm}
    \label{thm:main}
    CBS+TRDP is optimal and complete.
\end{thm}
\begin{proof}
    Per Lemma \ref{lem:elim}, CBS+TRDP only eliminates sub-optimal and infeasible solutions. Since CBS searches the CT in a best-first fashion, and optimal solutions cannot be eliminated by TRDP, CBS+TRDP is optimal. Per Lemma \ref{lem:finite}, TRDP makes the search space finite. Therefore, the size of the CT has a finite upper bound. Hence, CBS+TRDP is guaranteed to terminate, regardless of whether the problem instance is solvable.
\end{proof}

\section{Bypassing TRD Splits}

In settings where the number of agents is large compared to the size of the graph (agent-dense settings), many TRDs can occur, resulting in many $k$-way splits in the CT. This can lead to a very large CT and cause a significant amount of work. Our experiments showed that TRDP often reduces the depth in the CT at which a solution is found when compared to regular CBS in agent-dense settings. However, the average branching factor is increased. Thus TRDP usually causes a significantly larger CT to be generated when compared with regular CBS.

Fortunately, the $k$-way split can be completely avoided in many cases. The procedure to do so is simple and is based on the bypass procedure for regular CBS~\cite{CBSBP}. When a TRD is found, we test each agent individually for a bypass by adding a TRD constraint and re-planning its path. If the agent is able to find an alternate path that (1) avoids the loop, (2) does not increase the path cost, and (3) does not incur more conflicts, this path is a \textit{bypass}. Upon finding a bypass, we replace $\pi_i\in N.\Pi$ with the bypass in the parent node $N$. Then $N$ is re-inserted into the OPEN list. 
If we fail to find a bypass, the results of this computation (each $c_i$ and corresponding re-planned paths) are used to generate the $k$ successor nodes.

We found that this simple procedure, on average, avoided up to 99\% of $k$-way splits for agent-dense, solvable instances. We now prove that the TRDP bypass procedure preserves optimality and completeness.

\begin{thm}
    CBS+TRDP with the bypass procedure is optimal and complete.
\end{thm}
\begin{proof}
    Let $N$ be a CT node containing a TRD in $N.\Pi$. If a bypass is found, $N.\Pi$ is fixed with a new $\pi_i$ to avoid the TRD to create $N'$, and $N'$ is inserted into OPEN. Per the bypass procedure, the cost of $N'.\Pi$ is not increased nor decreased, hence does not change the optimality properties of CBS+TRDP. No constraints are added to $N'$, hence it is impossible to block any optimal or feasible solution in the sub-tree of $N'$. Therefore, the completeness properties of CBS+TRDP are preserved. In summary, adding the bypass procedure to CBS+TRDP retains optimality and completeness as shown in Theorem \ref{thm:main}.
\end{proof}

\section{Empirical Results}

We preface this section by reminding the reader that the primary intent of TRDP is to make CBS natively complete, not necessarily to scale to larger problems. Depending on the domain, the probability of encountering an unsolvable instance may be low. (None of the MAPF benchmark set contain unsolvable problem instances.) Additionally, the probability of encountering $k$-agent loops during the search may be low. In our investigation, we found that the probability of encountering TRDs during the search is rare in typical problem instances. When running all grid MAPF benchmark problems~\cite{stern2019multi}, less than 0.01\% of problem instances triggered the TRDP logic. The most common encounter of TRDs was in small mazes with corridor widths of 1 (an agent-dense setting). 

\begin{table}[!b]
\centering
\resizebox{\columnwidth}{!}{
\begin{tabular}{r|ccc}
Problem &
\begin{tikzpicture}[scale=.5]
\draw[step=1cm,color=gray] (0,0) grid (3,3);

\draw[fill=gray,color=gray] (1,0) rectangle (2,1);
\draw[fill=gray,color=gray] (1,2) rectangle (2,3);

\node[circle,draw=red,text=red,thick,fill=red!10,inner sep=0pt,minimum size=9pt] at (.5,.5) {};
\node[circle,draw=blue,text=blue,thick,fill=blue!10,inner sep=0pt,minimum size=9pt] at (2.5,.5) {};
\node[circle,draw=orange,text=blue,thick,fill=orange!10,inner sep=0pt,minimum size=9pt] at (.5,2.5) {};

\path (.6,2.5) edge [->,>=stealth,very thick,bend left=15,dashed,color=orange] node {} (2.6,.6);
\path (.5,.6) edge [->,>=stealth,very thick,bend left=15,dashed,color=red] node {} (2.5,2.5);
\path (2.4,.4) edge [->,>=stealth,very thick,bend left=15,dashed,color=blue] node {} (.4,2.4);

\end{tikzpicture} & 
\begin{tikzpicture}[scale=.5]
\draw[step=1cm,color=gray] (0,0) grid (3,3);

\draw[fill=gray,color=gray] (0,1) rectangle (1,3);
\draw[fill=gray,color=gray] (2,1) rectangle (3,3);

\node[circle,draw=red,text=red,thick,fill=red!10,inner sep=0pt,minimum size=10pt] at (.5,.5) {};
\node[circle,draw=blue,text=blue,thick,fill=blue!10,inner sep=0pt,minimum size=10pt] at (2.5,.5) {};
\node[circle,draw=orange,text=blue,thick,fill=orange!10,inner sep=0pt,minimum size=10pt] at (1.5,2.5) {};

\path (1.5,2.5) edge [->,>=stealth,very thick,dashed,color=orange] node {} (1.5,1.5);
\path (.5,.6) edge [->,>=stealth,very thick,bend left=15,dashed,color=red] node {} (2.5,.6);
\path (2.5,.4) edge [->,>=stealth,very thick,bend left=15,dashed,color=blue] node {} (.5,.4);

\end{tikzpicture} & 

\begin{tikzpicture}[scale=.5]

\node at (2.1,1.3) [rectangle,fill=white,inner sep=0pt] {\tiny 283};
\node at (.9,1.3) [rectangle,fill=white,inner sep=0pt] {\tiny 141};
\node at (1.3,2.1) [rectangle,fill=white,inner sep=0pt] {\tiny 283};
\node at (2.1,1.8) [rectangle,fill=white,inner sep=0pt] {\tiny 250};

\node[circle,draw=red,text=red,thick,fill=red!10,inner sep=0pt,minimum size=10pt] at (2.25,2.5) {};
\node[circle,draw=blue,text=blue,thick,fill=blue!10,inner sep=0pt,minimum size=10pt] at (2.5,.5) {};
\node[circle,draw=orange,text=blue,thick,fill=orange!10,inner sep=0pt,minimum size=10pt] at (.5,2.5) {};

\node[circle,draw=black,fill=black,inner sep=0pt,minimum size=3pt] at (1,1) {};
\node[circle,draw=black,fill=black,inner sep=0pt,minimum size=3pt] at (1.5,1.5) {};
\node[circle,draw=black,fill=black,inner sep=0pt,minimum size=3pt] at (.5,2.5) {};
\node[circle,draw=black,fill=black,inner sep=0pt,minimum size=3pt] at (2.25,2.5) {};
\node[circle,draw=black,fill=black,inner sep=0pt,minimum size=3pt] at (2.5,.5) {};
\path (1,1) edge [-,color=black] node {} (1.5,1.5);
\path (.5,2.5) edge [-,color=black] node {} (1.5,1.5);
\path (2.25,2.5) edge [-,color=black] node {} (1.5,1.5);
\path (2.5,.5) edge [-,color=black] node {} (1.5,1.5);

\path (.5,2.5) edge [->,>=stealth,very thick,bend left=15,dashed,color=orange] node {} (2.25,2.5);
\path (2.25,2.5) edge [->,>=stealth,very thick,bend left=15,dashed,color=red] node {} (2.5,.5);
\path (2.5,.4) edge [->,>=stealth,very thick,bend left=15,dashed,color=blue] node {} (.5,2.5);

\end{tikzpicture} \\
\midrule
No TRDP & 309 & 24 & 511 \\
With TRDP & 297 & 21 & 298 \\

\end{tabular}
}
\caption{Size of CT with and without TRDP.}
\label{tab:hand}
\end{table}


\begin{table}[!b]
\centering
\resizebox{\columnwidth}{!}{
\begin{tabular}{r|ccccc}
Problem &
\begin{tikzpicture}[scale=.5]
\draw[step=1cm,color=gray] (0,0) grid (1,2);

\node[circle,draw=red,text=red,thick,fill=red!10,inner sep=0pt,minimum size=10pt] at (.5,.5) {};
\node[circle,draw=blue,text=blue,thick,fill=blue!10,inner sep=0pt,minimum size=10pt] at (.5,1.5) {};

\path (.4,.5) edge [->,>=stealth,very thick,bend left=15,dashed,color=red] node {} (.4,1.5);
\path (.6,1.5) edge [->,>=stealth,very thick,bend left=15,dashed,color=blue] node {} (.6,.5);

\end{tikzpicture} & 
\begin{tikzpicture}[scale=.4]
\draw[step=1cm,color=gray] (0,0) grid (1,3);

\node[circle,draw=red,text=red,thick,fill=red!10,inner sep=0pt,minimum size=8pt] at (.5,.5) {};
\node[circle,draw=blue,text=blue,thick,fill=blue!10,inner sep=0pt,minimum size=8pt] at (.5,2.5) {};

\path (.4,.5) edge [->,>=stealth,very thick,bend left=15,dashed,color=red] node {} (.4,2.5);
\path (.6,2.5) edge [->,>=stealth,very thick,bend left=15,dashed,color=blue] node {} (.6,.5);

\end{tikzpicture} & 
\begin{tikzpicture}[scale=.3]
\draw[step=1cm,color=gray] (0,0) grid (1,4);

\node[circle,draw=red,text=red,thick,fill=red!10,inner sep=0pt,minimum size=6pt] at (.5,.5) {};
\node[circle,draw=blue,text=blue,thick,fill=blue!10,inner sep=0pt,minimum size=6pt] at (.5,3.5) {};

\path (.4,.5) edge [->,>=stealth, thick,bend left=15,dashed,color=red] node {} (.4,3.5);
\path (.6,3.5) edge [->,>=stealth, thick,bend left=15,dashed,color=blue] node {} (.6,.5);

\end{tikzpicture} & 
\begin{tikzpicture}[scale=.25]
\draw[step=1cm,color=gray] (0,0) grid (1,5);

\node[circle,draw=red,text=red,thick,fill=red!10,inner sep=0pt,minimum size=5pt] at (.5,.5) {};
\node[circle,draw=blue,text=blue,thick,fill=blue!10,inner sep=0pt,minimum size=5pt] at (.5,4.5) {};

\path (.4,.5) edge [->,>=stealth,thick,bend left=15,dashed,color=red] node {} (.4,4.5);
\path (.6,4.5) edge [->,>=stealth,thick,bend left=15,dashed,color=blue] node {} (.6,.5);

\end{tikzpicture} &
\begin{tikzpicture}[scale=.5]
\draw[step=1cm,color=gray] (0,0) grid (3,3);

\draw[fill=gray,color=gray] (1,0) rectangle (2,1);
\draw[fill=gray,color=gray] (1,2) rectangle (2,3);

\node[circle,draw=red,text=red,thick,fill=red!10,inner sep=0pt,minimum size=9pt] at (.5,.5) {};
\node[circle,draw=blue,text=blue,thick,fill=blue!10,inner sep=0pt,minimum size=9pt] at (2.5,.5) {};
\node[circle,draw=orange,text=blue,thick,fill=orange!10,inner sep=0pt,minimum size=9pt] at (.5,2.5) {};
\node[circle,draw=green,text=blue,thick,fill=green!10,inner sep=0pt,minimum size=9pt] at (2.5,2.5) {};

\path (.6,2.5) edge [->,>=stealth,very thick,bend left=15,dashed,color=orange] node {} (2.6,.6);
\path (.5,.6) edge [->,>=stealth,very thick,bend left=15,dashed,color=red] node {} (2.5,2.5);
\path (2.4,.4) edge [->,>=stealth,very thick,bend left=15,dashed,color=blue] node {} (.4,2.4);
\path (2.6,2.4) edge [->,>=stealth,very thick,bend left=15,dashed,color=green] node {} (.6,.4);

\end{tikzpicture} \\
\midrule
CT Nodes & 5 & 36 & 541 & 27K & 357K \\

\end{tabular}
}
\caption{Size of CT for unsolvable MAPF instances.}
\label{tab:hand}
\end{table}

Fortunately, although problem instances which actually exercise TRDP logic may be relatively rare in practice, its inclusion in CBS has many benefits. Foremost, it ensures completeness for CBS in both MAPF and MAPF\textsubscript{Q} settings. Second, TRDP has the ability to prune portions of the high-level search which are unsolvable. Finally, including TRDP it is relatively inexpensive. The complexity of detecting TRDs in a solution is $O(kC^2)$ where $C$ is the makespan of a solution. The expense is normally mitigated further because the occurrence of single-agent loops is also relatively rare, meaning the TRD check usually exits early (see Algorithm \ref{alg:trdd} line 3 and line 5).

In this section we show two things: (1) that TRDP has performance benefits for hand-crafted examples with high agent density and (2) that TRDP incurs no significant cost for MAPF instances in general.

We ran CBS with TRDP for the entire set of MAPF benchmarks on 4-, 8-, 16- and 32-neighbor grid domains~\cite{neighborhoods}. Note that 4-neighbor grids are classic MAPF instances, but the other domains with higher connectivity are MAPF\textsubscript{Q} instances. Our tests start with two agents and increase the number of agents by one until the problem instance is no longer solvable in under 30 seconds. We recorded the runtimes and the total number of agents for which we could solve within the 30 second time limit. Additionally, we counted the number of times that TRDs occurred in any CT node in this process.

Over the nearly 32,000 experiments and millions of CT nodes generated, only 29 TRDs were ever encountered. Most TRDs were found to occur in settings with narrow corridors. These have high agent density per the number of traversable edges. As the ratio of agents per graph vertices and edges increases, the probability of encountering TRDs increases. Thus, part of our focus in this section is on agent-dense problem instances.

For some instances with high agent density, using TRDP can help prune parts of the CT in CBS which contain $k$-agent loops. A worked-out example of such a scenario is shown in Figure \ref{fig:incomplete}. Table \ref{tab:hand} shows the reduction in CT nodes required to solve some hand-crafted problem instances. The first two instances are for unit-cost grids, and the last instance is for a weighted graph with fixed wait actions of 20. We see that in all three instances, some nodes are pruned from the CT, saving work before the solution is found.

Table \ref{tab:time} shows the amount of total runtime and the amount of runtime attributed to TRD checking when varying the map type, number of agents and branching factor. The top half of the table is for ``classic'' 4-neighbor grids, and the bottom half of the table is for 32-neighbor grids. Note that in all cases, the exact same solution was found with TRDP turned on, and TRDP did not prune any nodes from the CT. We see a trend that TRDP increases the runtime in most cases, but the increase is less than .01\% on average for 4-neighbor grids and less than about .01\% on average for 32-neighbor grids. But the overhead was up to about 3\% in the most extreme cases. We see that the amount of increase is positively correlated to (1) the number of agents $k$, (2) the map size (which results in larger mean values for $C$) and (3) the single-agent branching factor. In the larger branching factor settings (e.g., in 32-neighbor grids) many edges are longer, typically resulting in longer duration actions. This causes TRD checking to take longer because of actions having partial time overlap. Additionally, edge-edge crossings in these graphs are abundant, causing more conflicts in general, resulting in a larger CT. In general, as the size of the CT increases, agents have more constraints, causing longer paths, increasing $C$, and increasing the number of single-agent loops. These factors cause the TRD check time to increase.

\begin{table}[!t]
\centering
\resizebox{\columnwidth}{!}{
\begin{tabular}{lc|ll|ll}

\multicolumn{6}{c}{4-Neighbor Grid} \\
\midrule
 & & \multicolumn{2}{c}{4 Agents} & \multicolumn{2}{c}{10 Agents} \\
\multicolumn{1}{c}{Map} & Size & \multicolumn{1}{c}{Total} & \multicolumn{1}{c}{TRD} & \multicolumn{1}{c}{Total} & \multicolumn{1}{c}{TRD} \\
 \midrule
Berlin\_1\_256 & 256x256 & \;\;\;236$\pm$26 & 0$\pm$4 & \;\;\,508$\pm$119 & \;1$\pm$13  \\
brc202d & 481x530 & 1,389$\pm$597 & 4$\pm$11 & 9,451$\pm$1,128& \;9$\pm$14 \\
empty-8-8 & 8x8 & \;\;\;\;\;\;\,1$\pm$2 & 0$\pm$0 & \;\;\;\;\,23$\pm$29 & \;0$\pm$0 \\
ht\_chantry & 141x162 & \;\;\;\;\;95$\pm$150 & 1$\pm$6 & \;\;\,856$\pm$2,180 & \;0$\pm$6 \\
maze-128-128-10 & 128x128 & \;\;\;337$\pm$1,164 & 2$\pm$16 & 8,579$\pm$6,347 & \;1$\pm$11 \\
random-32-32-10 & 32x32 & \;\;\;\;\;\;\,6$\pm$2 & 0$\pm$2 & \;\;\;\;\,17$\pm$14 &\;0$\pm$4 \\
room-64-64-8 & 64x64 & \;\;\;318$\pm$1,186 & 1$\pm$4 & 4,143$\pm$6,036 & \;4$\pm$22 \\

\midrule
\; \\
\multicolumn{6}{c}{32-Neighbor Grid} \\
\midrule
 & & \multicolumn{2}{c}{4 Agents} & \multicolumn{2}{c}{10 Agents} \\
\multicolumn{1}{c}{Map} & Size & \multicolumn{1}{c}{Total} & \multicolumn{1}{c}{TRD} & \multicolumn{1}{c}{Total} & \multicolumn{1}{c}{TRD} \\
 \midrule
Berlin\_1\_256 & 256x256 & \;\;\,775$\pm$430 & \;\,1$\pm$14 & \;\,1,484$\pm$20 & \;\;21$\pm$98 \\
brc202d & 481x530 & 1,929$\pm$405 & 24$\pm$46 & \;\,3,902$\pm$6,053 & 131$\pm$205 \\
empty-8-8 & 8x8 & 2,404$\pm$4,137 & \;\;0$\pm$1  & \;\,5,120$\pm$5,128 & \;\;\;\,2$\pm$7 \\
ht\_chantry & 141x162 & 1,980$\pm$3,198 & \;\;6$\pm$11 & \;\,6,836$\pm$7,102 & \;\,11$\pm$27 \\
maze-128-128-10 & 128x128 & 3,183$\pm$7,640 & 26$\pm$49 & 17,635$\pm$9,240 & 110$\pm$130 \\
random-32-32-10 & 32x32 & \;\;\;\;\,16$\pm$4 & \;\;0$\pm$5 & \;\,1,445$\pm$1,421 & \;\;\;\,5$\pm$16 \\
room-64-64-8 & 64x64 & 4,689$\pm$5,005 & \;\;5$\pm$12 & 21,463$\pm$9,907 & \;\,33$\pm$78 \\

\end{tabular}
}
\caption{Mean and standard deviation of total runtime and portion of runtime attributed to TRD duplicate detection in milliseconds.}
\label{tab:time}
\end{table}

\section{Conclusion}

We have introduced temporally-relative duplicate pruning (TRDP), a technique for decoupled MAPF and MAPF\textsubscript{Q} algorithms for guaranteeing completeness. TRDP solves a long-standing problem for CBS by making CBS natively complete. We have also shown theoretically and empirically that TRDP has desirable properties for MAPF algorithms, namely, increased efficiency for agent-dense settings and completeness guarantees. TRDP adds these benefits while adding no significant computational overhead.

\bibliography{main}

\end{document}